\documentclass[conference]{IEEEtran}
\usepackage{amsmath,amsfonts,amssymb,amsthm}
\usepackage{times}
\usepackage{comment}
\usepackage{xcolor}
\usepackage{graphicx}
\usepackage{subfigure}
\usepackage{multirow}
\usepackage[numbers]{natbib}
\usepackage{multicol}
\usepackage[bookmarks=true]{hyperref}
\usepackage{makecell} 

\usepackage{floatrow}
\usepackage[linesnumbered,ruled]{algorithm2e}
\usepackage{txfonts}
\usepackage{mathtools}

\newcommand{\bA}{\boldsymbol{A}}
\newcommand{\bF}{\boldsymbol{F}}
\newcommand{\br}{\boldsymbol{r}}
\newcommand{\bw}{\boldsymbol{w}}
\newcommand{\bxi}{\boldsymbol{\xi}}
\newtheorem{thm}{Theorem}[]
\newtheorem{problem}[thm]{Problem}
\newtheorem{lemma}[thm]{Lemma}

\begin{document}

\title{Gait design for limbless obstacle aided locomotion using geometric mechanics}


\author{\authorblockN{Baxi Chong\authorrefmark{1},
Tianyu Wang\authorrefmark{2}, 
Daniel Irvine\authorrefmark{3}, 
Velin Kojouharov\authorrefmark{2}, 
Bo Lin\authorrefmark{3}, \\
Howie Choset\authorrefmark{4}, 
Daniel I. Goldman\authorrefmark{1}, 
and Grigoriy Blekherman\authorrefmark{3}}
\authorblockA{\authorrefmark{1}School of Physics \authorrefmark{2}School of Mechanical Engineering \authorrefmark{3}School of Mathematics, Georgia Institute of Technology}
\authorblockA{\authorrefmark{4}Robotics Institute, Carnegie Mellon University}}

\maketitle

\begin{abstract}
Limbless robots have the potential to maneuver through cluttered environments that conventional robots cannot traverse. As illustrated in their biological counterparts such as snakes and nematodes, limbless locomotors can benefit from interactions with obstacles, yet such obstacle-aided locomotion (OAL) requires properly coordinated high-level self-deformation patterns (gait templates) as well as low-level body adaptation to environments. Most prior work on OAL utilized stereotyped traveling-wave gait templates and relied on local body deformations (e.g., passive body mechanics or decentralized controller parameter adaptation based on force feedback) for obstacle navigation, while gait template design for OAL remains less studied. In this paper, we explore novel gait templates for OAL based on tools derived from geometric mechanics (GM), which thus far has been limited to homogeneous environments. Here, we expand the scope of GM to obstacle-rich environments. Specifically, we establish a model that maps the presence of an obstacle to directional constraints in optimization. In doing so, we identify novel gait templates suitable for sparsely and densely distributed obstacle-rich environments respectively. Open-loop robophysical experiments verify the effectiveness of our identified OAL gaits in obstacle-rich environments. We posit that when such OAL gait templates are augmented with appropriate sensing and feedback controls, limbless locomotors will gain robust function in obstacle rich environments.
\end{abstract}

\IEEEpeerreviewmaketitle
\section{Introduction}

Elongate limbless locomotors have advantages in navigating cluttered and confined spaces. For instance, adaptation to cluttered environments is believed to be a source of evolutionary pressure for limblessness in Squamates (lizards and snakes)~\cite{rieppel1988review,simoes2015visual}. In order to move through such cluttered environments, these animals had evolved the capability to push off their surroundings to locomote. This is commonly known as obstacle-aided locomotion (OAL)~\cite{liljeback2011experimental,majmudar2012experiments,kelley1997effects}. Moreover, many biological limbless locomotors can have higher speeds with OAL than in obstacle-free environments~\cite{majmudar2012experiments,kelley1997effects} while legged locomotors often slow down as heterogeneity increases~\cite{sponberg2008neuromechanical,collins2013rock,parker2016effects,gast2019preferred}. Unfortunately, it is still challenging for elongated limbless robots to approach the performance of their biological counterparts displayed in OAL. 

To replicate the successful biological OAL in robotic/artificial counterparts, prior work has considered OAL in robotic applications. \citet{transeth2008snake} built physical models of robot-obstacle interactions, and made quantitative predictions of robot locomotion in obstacle-rich environments. \citet{liljeback2011experimental} noted that the interactions between robots and obstacles are only useful when the force from the obstacle to robots aligns with the desired direction of motion. Specifically the ``beneficial obstacle" and ``detrimental obstacle" are distinguished based on their configuration relative to the robot. Recently, compliant control (shape-based compliance) was also introduced to improve the performance of limbless robots among obstacles~\cite{travers2018shape,wang2020directional,wang2022omega}. 

\begin{figure}[t]
\centering
\includegraphics[width=1\linewidth]{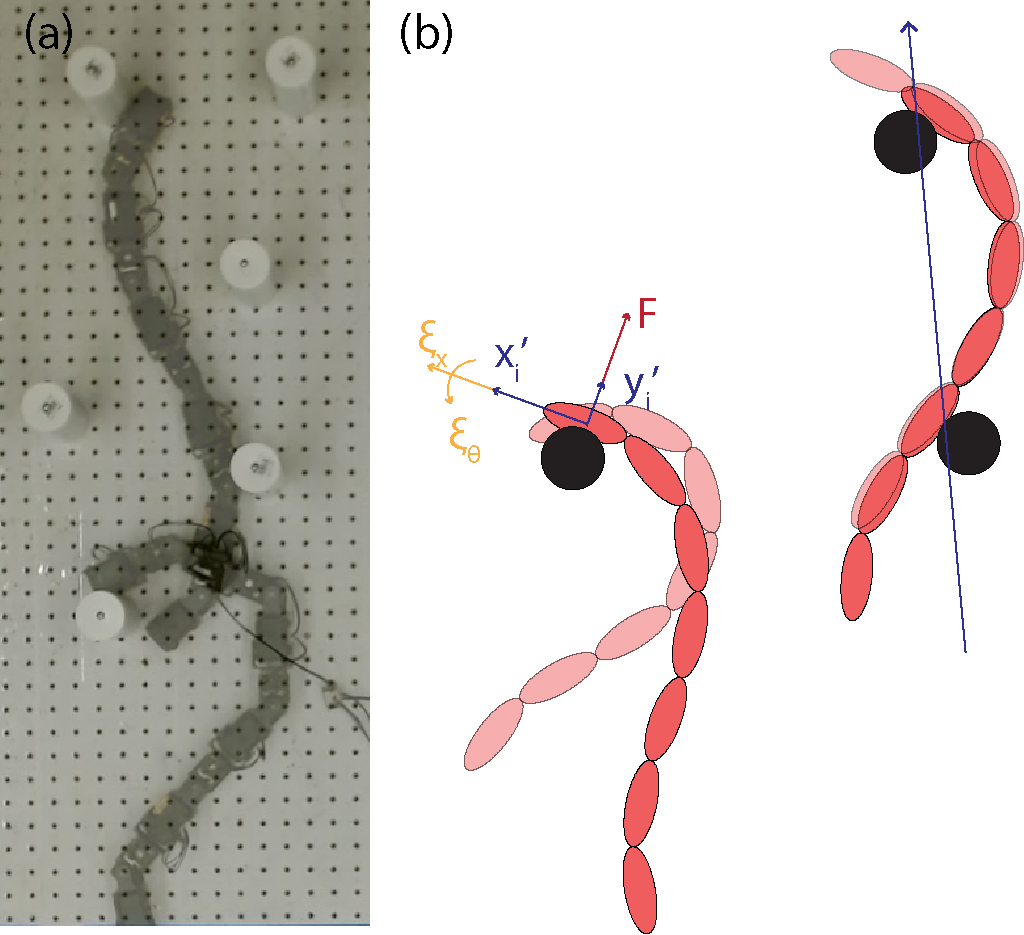}
\caption{\textbf{A robophysical and theoretical model of obstacle aided locomotion} (a) Top view of the robophysical model navigating among multiple obstacles. (b) The theoretical model for obstacle aided locomotion with (\textit{left}) a single obstacle and (\textit{right}) multiple obstacles.
}
\label{fig:intro}
\end{figure}


As suggested in prior work, gait template design\footnote{A periodic internal shape change that causes a net displacement in the world reference frame.} is crucial to the performance of limbless locomotors~\cite{gong2016simplifying,marvi2014sidewinding,chong2021frequency}. Appropriate gait templates can greatly simplify the control/adaptation of robots especially in heterogeneous environments~\cite{full1999templates,katz2019mini,ozkan2020systematic,astley2015modulation,chong2021frequency}. Most limbless robots use traveling-wave gait templates for locomotion where sinusoidal oscillation of body joint bending propagates from head to tail under constant amplitude (i.e., phase modulation)~\cite{hirose1993biologically,crespi2005swimming,schiebel2019mechanical}. To the best of our knowledge, most OAL work has focused on force-feedback decentralized adaptation of traveling-wave gait templates to interact with obstacles, where the choices of gait templates are often pre-determined~\cite{liljeback2011experimental,travers2018shape,wang2020directional}. In other words, there is a lack of gait templates designed (other than traveling-wave) specifically for obstacle aided locomotion. 

Geometric mechanics (GM) is a framework for gait design. GM was developed to study swimming in obstacle-free low Reynolds number fluids~\cite{purcell1977life,wilczek1989geometric}. Recent work has shown that GM can also offer insights in gait design in terrestrial contexts (e.g., granular media and frictional ground) where frictional forces dominate over inertial forces~\cite{chong2021coordination,chong2021frequency}. In GM the motion of a locomotion system is separated into a shape space (the internal joint angle space) and a position space (position and orientation of locomotor in the world frame). 
By establishing the mapping between velocities in shape and position spaces, GM offers tools that allow us to visually analyze, design and optimize gaits~\cite{hatton2013geometric}. Although GM has produced a number of highly effective gait templates, prior work in GM has been limited to obstacle-free environments.

\begin{figure}[t]
\centering
\includegraphics[width=1\linewidth]{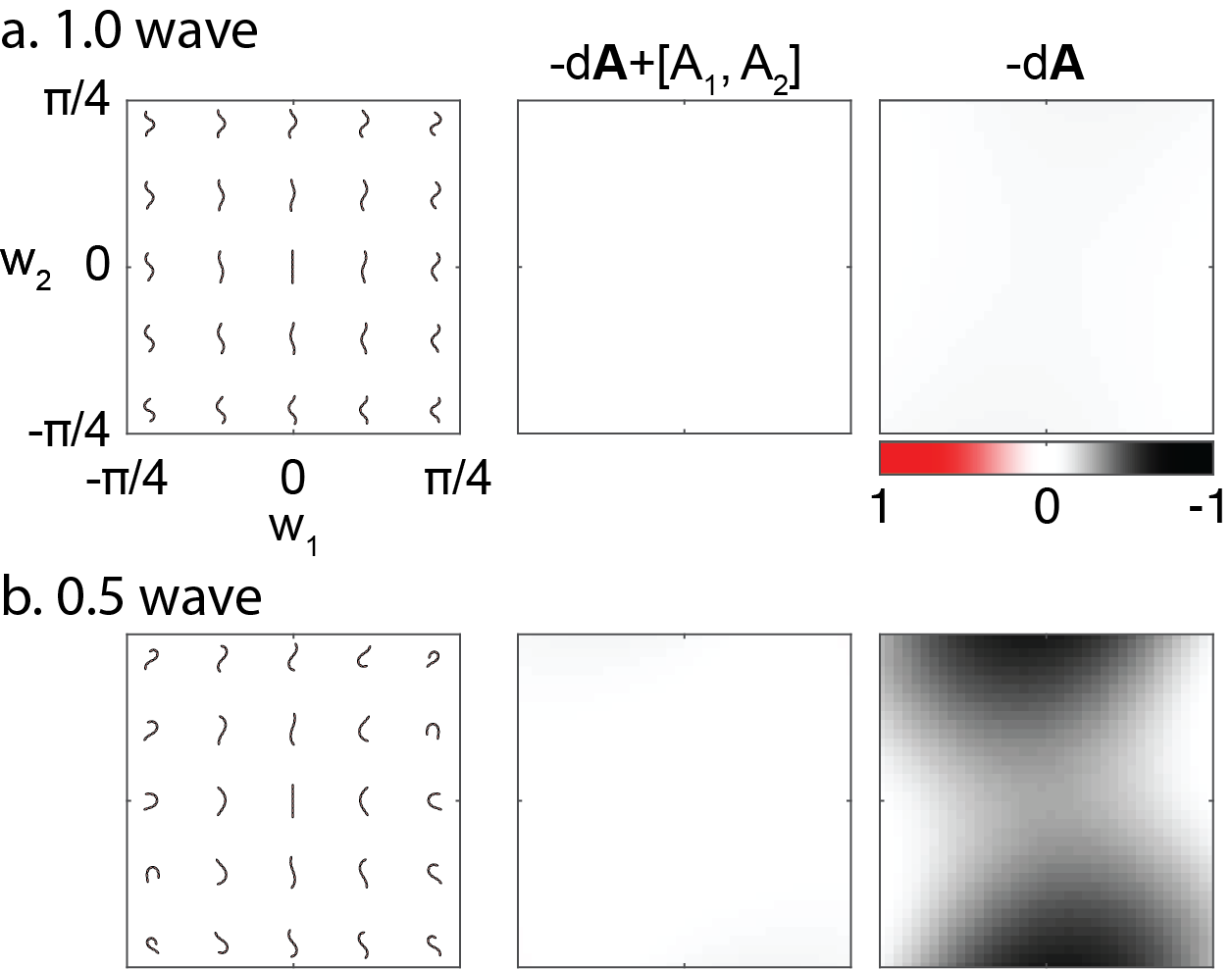}
\caption{\textbf{Forward velocity integral and Lie bracket effect} From left to right: the shape space, the height function ($-\textbf{d}\bA+\boldsymbol{[A_1,\ A_2]}$), and the forward velocity integral ($-\textbf{d}\bA$). We compared two wave numbers: (a) $f_s=1$ and (b) $f_s=0.5$. In both cases, the height function has zero values over the shape space. Notably, for $f_s=1$, neither forward velocity integral nor Lie bracket effect has significant contributions to forward displacement. In contrast, when $f_s=0.5$, the forward velocity integral and Lie bracket effect have non-negligible, opposite contribute to forward displacement. The color bar scale axes labeling are identical in all panels.}
\label{fig:NumOfWaves}
\end{figure}

In this paper, we seek to expand the scope of GM to obstacle-rich\footnote{Here, we consider obstacles as vertical posts randomly distributed on flat terrains.} environments. Challenges of extending GM to design gait templates in heterogeneous environments include but are not limited to (1) modeling the interaction between obstacle and robot, (2) mapping the presence of obstacle from position space to constraints in shape space, and (3) identifying whether the obstacle (at a given position relative to the robot) is beneficial. We establish a new physical robot-obstacle interaction model integrating the presence of an obstacle into the GM framework. In doing so, we then convert the gait design problem into a discrete optimization problem in graphs. As a result, we identify elliptical gait templates which combine both amplitude modulation and phase modulation, specialized for navigating sparsely-distributed obstacle-rich environments. Further, we confirm that traveling-wave gait templates are specialized for densely-distributed obstacle-rich environments, which is consistent with prior works~\cite{liljeback2011experimental,travers2018shape}. We verify our results using a robophysical model (Fig. \ref{fig:intro}). 

\section{Geometric mechanics}

In this subsection, we provide an overview of the geometric tools that undergird the analysis framework introduced in this paper.
For a more detailed and comprehensive review, we refer readers to~\cite{chong2019hierarchical,toruspaper,Marsden,zhong2018coordination}.

\subsection{Kinematic Reconstruction Equation}

In systems where inertial effects are negligible, the equations of motion (\cite{Marsden}) can be approximated as:

\begin{equation}\label{eq:EquationOfMotion1}
    \bxi={\bA_r(\br)\dot \br},
\end{equation}

where $\bxi=[\xi_x, \xi_y, \xi_\theta]$ denotes the body velocity in the forward, lateral, and rotational directions; ${\br}$ denotes the internal shape variables (joint angles); $\bA_r(\br)$ is the local connection matrix, which encodes environmental constraints and the conservation of momentum. The analysis and visualization power of geometric mechanics is particularly effective when the shape variable is 2-dimensional, i.e., $\br\in \mathbb{R}^2$. In the applications where there are more than 2 joints (e.g. $N$ degrees-of-freedom), we use two shape basis functions~\cite{gong2016simplifying} to reduce the dimensionality of the system:
\begin{equation}
     \br = \left[\boldsymbol{\beta}_{1},\ \boldsymbol{\beta}_{2}\right]\bw, \ \  \bxi={\bA_r\Big(\br(\bw)\Big)\dot \bw}  = \bA(\bw)\dot \bw\label{eq:EquationOfMotion2}
\end{equation}

\noindent where $\boldsymbol{\beta}_{1},\ \boldsymbol{\beta}_{2}\in\mathbb{R}^{N}$ are shape basis functions, $\bw\in\mathbb{R}^{2}$ is the reduced shape variable, and $\bA$ is the local connection matrix expressed with respect to reduced shape variables. In applications to limbless robots with $N$ joints, the shape basis functions are often chosen to be:

\begin{equation}\label{eq:shapebasis}
    \boldsymbol{\beta}_{1}(i) = \sin\left(2\pi f_s\frac{i}{N-1}\right),\ \ \boldsymbol{\beta}_{2}(i) = \cos\left(2\pi f_s\frac{i}{N-1}\right)
\end{equation}
    
\noindent where $f_s$ is the number of spatial waves, $i$ denotes the joint index.


\subsection{Numerical Derivation of the Local Connection Matrix}

The local connection matrix $\bA$ can be numerically derived using resistive force theory (RFT) to model the ground reaction force~\cite{li2013terradynamics,sharpe2015locomotor,zhang2014effectiveness}. Specifically, the ground reaction force (GRF) experienced by the locomotor is the sum of the GRF experienced by each body segment. RFT decomposes the GRF experienced by a body segment of a locomotor into two components: $\bF_{\parallel}$ and $\bF_{\perp}$, reaction force along the direction parallel and perpendicular to the body segment respectively. 

From geometry and physics of GRF, reaction forces of each segment can be calculated from the body velocity $\bxi$, reduced body shape $\bw$, and reduced shape velocity $\dot{\bw}$~\cite{rieser2019geometric,murray2017mathematical}.
Assuming quasi-static motion, we consider the total net force applied to the system is zero at any instant in time:

\begin{equation}\label{eq:forceIntegral}
    \bF=\sum_{i=1}^{N} {\left[\bF^{i}_{\parallel}\left(\bxi,\bw,\dot{\bw}\right)+\bF^{i}_{\perp}\left(\bxi,\bw,\dot{\bw}\right)\right]}=0.
\end{equation}

At a given body shape $\bw$, Eq.(\ref{eq:forceIntegral}) connects the shape velocity $\dot{\bw}$ to the body velocity $\bxi$. Therefore, by the implicit function theorem and the linearization process, we can numerically derive the local connection matrix $\bA(\bw)$.
In our implementation, we compute the solution of Eq.(\ref{eq:forceIntegral}) using the MATLAB function \textit{fsolve}.

\subsection{Connection Vector Fields and Height Functions}

Each row of the local connection matrix $\bA$ corresponds to a component direction of the body velocity. Each row of the local connection matrix, over the shape space, then forms a connection vector field.
In this way, the body velocities in the forward, lateral, and rotational directions are computed as the dot product of connection vector fields and the reduced shape velocity~$\dot{\bw}$. The displacement along the gait path $\partial \phi$ can be obtained by integrating the ordinary differential equation below~\cite{hatton2015nonconservativity}:
\begin{equation}\label{eq:fullapprox}
    g(T) = \int_{\partial \phi} {T_eL_{g(\bw)} \bA(\bw) \mathrm{d}{\bw}},
\end{equation}

\noindent where $g(\bw)= [x(\bw), y(\bw), \alpha(\bw)]$ represents the position and rotation of the body frame viewed in the world frame at position $\bw$; $T$ is the time period of a gait cycle; and $T_eL_{g}$ is the differential at the identity of the left multiplication map $L_g\colon SE(2) \to SE(2)$, i.e.

\begin{equation*}
    T_eL_{g}=\begin{bmatrix} 
        \cos(\alpha) & -\sin(\alpha) & 0  \\
        \sin(\alpha) & \cos(\alpha) & 0  \\
        0 & 0 & 1
    \end{bmatrix}.
\end{equation*}
\noindent The group element $g = (x,y,\theta)\in \mathrm{SE}(2)$ represents the position and rotation of the center of mass of the robot. Hence $g(T)=[\Delta x,\Delta y, \Delta \theta]$ computes the translation and rotation of the body frame (w.r.t. the world frame) in one gait cycle.
The forward velocity integral can therefore provide a first-order approximation to Eq.~\ref{eq:fullapprox}:

\begin{equation}
    \begin{pmatrix} 
        \Delta x \\
        \Delta y \\
        \Delta \theta 
    \end{pmatrix}
    \approx \int_{\partial \phi} {\bA(\bw)\mathrm{d}\bw} =  \int_{\partial \phi} {\begin{bmatrix}
        \bA^x(\bw) \\ \bA^y(\bw) \\ \bA^\theta (\bw)
    \end{bmatrix}\mathrm{d}\bw},
    \label{eq:lineint}
\end{equation}

\noindent where $\bA^x, \bA^y, \bA^\theta$ are the three rows of the local connections respectively. According to Stokes' Theorem, the line integral along a closed curve $\partial \phi$ is equal to the surface integral of the curl of $\bA(\bw)$ over the surface enclosed by $\partial \phi$:
\begin{equation}\label{eq:stokes}
    \int_{\partial \phi} {\bA(\bw)\mathrm{d}\bw}=\iint_{\phi} {-\textbf{d}\bA(\bw)\mathrm{d}w_1\mathrm{d}w_2},
\end{equation}
where $\phi$ denotes the surface enclosed by $\partial \phi$, $-\textbf{d}\bA(\bw)$ denotes the curl of the connection vector field.

Note that in the simplification from Eq.~\ref{eq:fullapprox} to Eq.~\ref{eq:lineint}, the forward displacement is approximated by the direct integration of forward speed. In reality, the combination of lateral and rotational velocities can lead to net translation in the forward direction. For example, car undergoing parallel parking will have zero instantaneous lateral velocity but can have finite lateral displacement with properly sequenced forward and rotational velocity. Such effect is known as Lie bracket effect~\cite{hatton2015nonconservativity} and is neglected in Eq.~\ref{eq:fullapprox}. The first order of Lie bracket effect can be compensated for by introducing a Lie bracket correction term~\cite{hatton2015nonconservativity}. Higher order Lie bracket effects can be minimized by properly choosing the body frame~\cite{linoptimizing}. With the Lie bracket correction term, we can better approximate the net forward displacement~\cite{hatton2015nonconservativity}:
\begin{align}\label{eq:hfintro}
    g(T)
    &= \iint_{\phi} \Big( \underbrace{-\textbf{d}\bA(\bw)+\boldsymbol{[A_1,\ A_2]}}_{D\bA(\bw)} \Big)\  \mathrm{d}w_1\mathrm{d}w_2 
\end{align}
The three rows of $D\bA(\bw)$ can thus produce three height functions in the forward, lateral, and rotational directions respectively. 

\begin{figure}[t]
\centering
\includegraphics[width=1\linewidth]{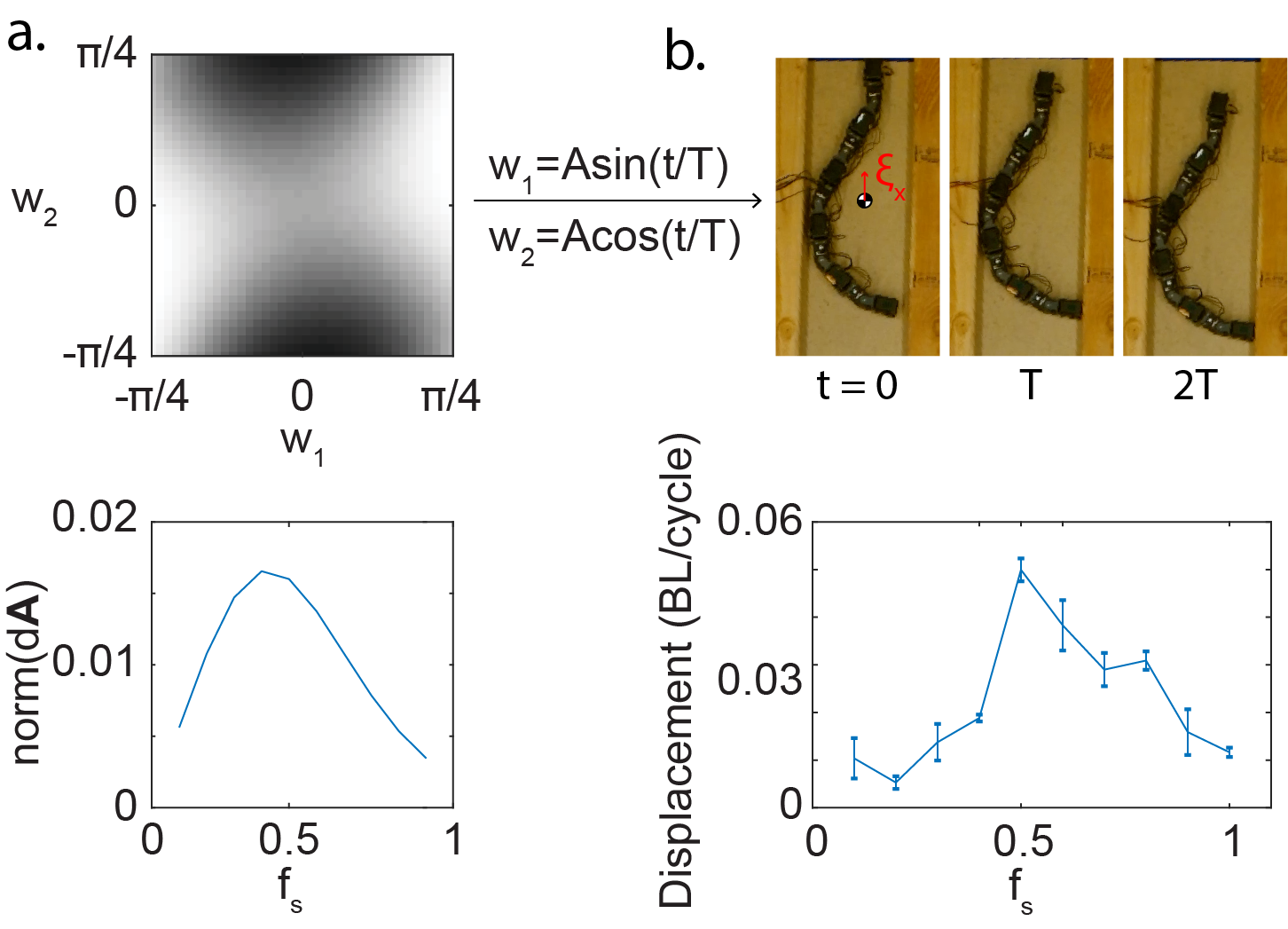}
\caption{\textbf{Lie bracket effect} (a) Theoretical prediction of forward velocity integral by calculating the Frobenius Norm of $\textbf{d}\bA(\bw)$ for a range of spatial frequencies ($f_s$). (b) Experimental verification. We used a pair of smooth parallel walls to restrict the robot's body velocity only in forward direction. We tested the locomotion performance for gaits with different spatial frequencies. Backwards locomotion is observed with its peak at $f_s=0.5$, which is consistent with our theoretical predictions.}
\label{fig:lie}
\end{figure}

\subsection{Lie bracket effect for OAL}

Limbless locomotors have limited mobility on hard ground ~\cite{liljeback2013lateral,chongmoving,alben2019efficient}. From a geometric perspective, we posit that some symmetry exists to limit the mobility of limbless locomotion on hard ground. The presence of obstacles can break such symmetry and therefore facilitate effective locomotion. To explore such symmetry breaking, in Fig.~\ref{fig:NumOfWaves}, we compared the height function for limbless robot with different shape basis functions ($f_s = \{ 1,\ 0.5\}$ in Eq.~\ref{eq:shapebasis}). In both cases, the height function $ D\bA(\bw)$ is almost constantly zero over the entire shape space, indicating that the robot has almost negligible speed regardless of the choice of gait (Fig.~\ref{fig:NumOfWaves}). When we look carefully at different components of the height function, we notice that robots with 1 wave ($f_s = 1$) and 0.5 waves ($f_s = 0.5$) exhibit distinct properties. On the one hand, for the robot with 1 wave, neither the forward velocity integral nor the Lie bracket effect can lead to significant translation (Fig. \ref{fig:NumOfWaves}.a). On the other hand, for the robot with 0.5 waves, the forward velocity integral and the Lie bracket effect have the same magnitude but opposite direction contribution to locomotion (Fig. \ref{fig:NumOfWaves}.b). This observation indicates lateral forces (likely from obstacles) can also contribute to forward velocities when the limbless locomotor is operating at appropriate spatial wave numbers $f_s$.

To determine the $f_s$ that can benefit the most from lateral forces, we calculated the Frobenius norm\footnote{We chose Frobenius norm to approximate the magnitude of the vector field} of $\textbf{d}\bA(\bw)$ for a range of wave numbers (Fig.~\ref{fig:lie}). From the geometric analysis, we predict that wave number $\approx$ 0.5 will have the highest possibility to benefit from lateral forces, and specialized in OAL.

\begin{figure}[t]
\centering
\includegraphics[width=1\linewidth]{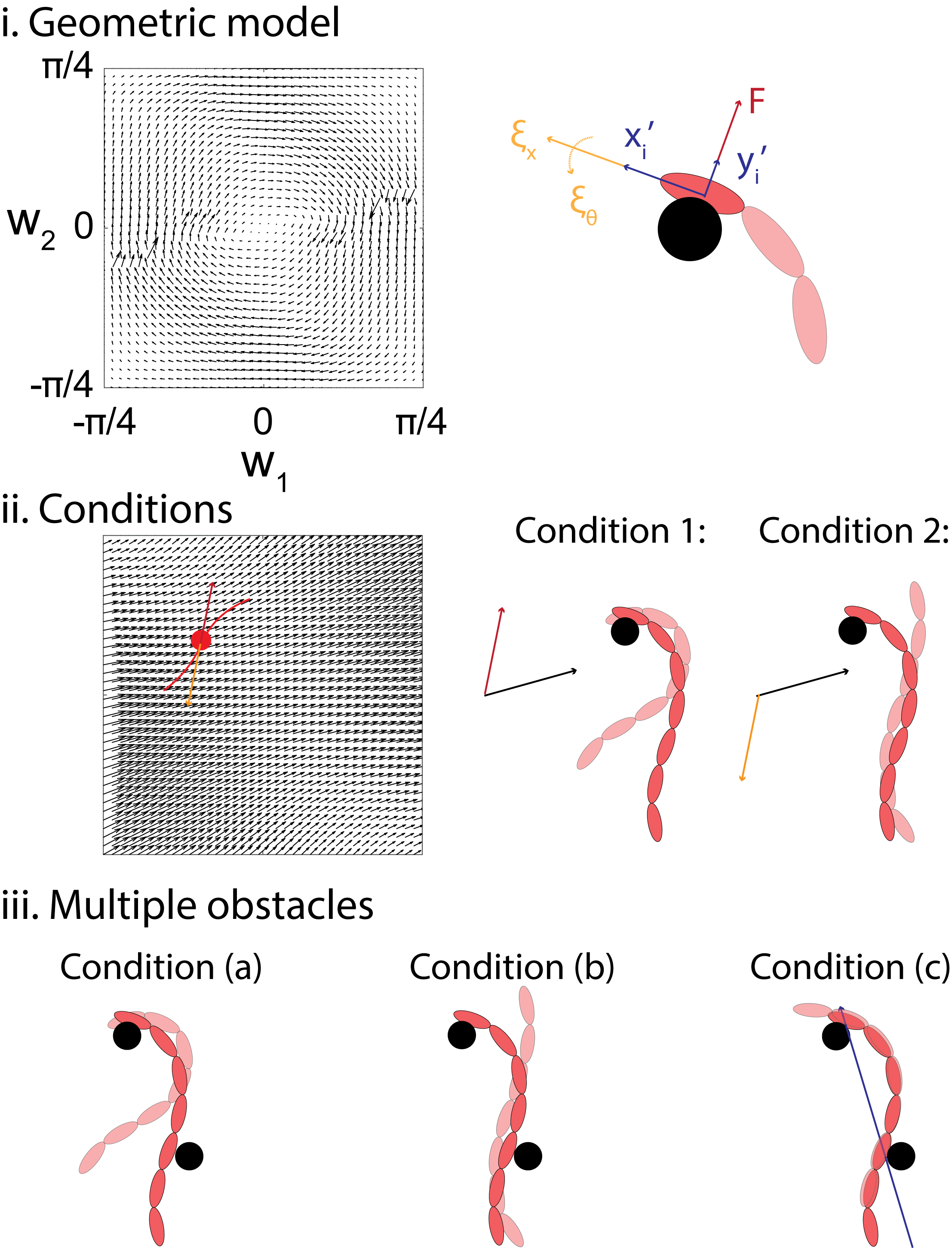}
\caption{\textbf{Modeling interactions between robot and obstacles} (i) (\textit{Left}) The vector field $V_1$ assuming the obstacle has interactions with the head link ($i_o=1$). (\textit{Right}) Force relationship illustrations for interactions between robot and obstacle. (ii) (\textit{Left}) The vector field $V_2$ assuming the obstacle has interactions with the head link ($i_o=1$). (\textit{Right}) The two conditions in Sec. \ref{sec:model}. (iii) OAL with multiple obstacles. Three conditions are compared. Note that in condition (c), obstacles constrain the lateral and rotational oscillation of robot's central body axis (blue arrow).}
\label{fig:VF12}
\end{figure}

\section{Modeling Interaction with One Obstacle}\label{sec:model}

\subsection{Geometric Model}
In the previous section, we introduced a derivation of the local connection vector field in homogeneous environments. In heterogeneous environments, the interactions with obstacles can often lead to changes in force and torque balance, and thus changes in the connection vector field. In this section, we establish a new method to numerically calculate the connection vector field, respecting the interactions between the robot and obstacles in its environment. Note that to simplify our analysis, we assume that the friction between the robot and the obstacle is negligible~\cite{liljeback2011experimental}.

Consider one obstacle in contact with the robot. Index $i_0$ denotes the link of contact. We further assume that $i_0$ does not change in each obstacle-interaction instance. This assumption is later justified in robot experiments.

For simplicity, our analysis below assumes that the obstacle resides on the left hand side (LHS) of link $i_0$. The analysis for the right hand side (RHS) obstacle will be symmetric to our analysis below. Existence of the obstacle will restrict the lateral body velocity~${\xi_y}\geq0$. In this way, there are two mutually exclusive conditions for the lateral body velocity:
\begin{figure}[t!]
\centering
\includegraphics[width=1\linewidth]{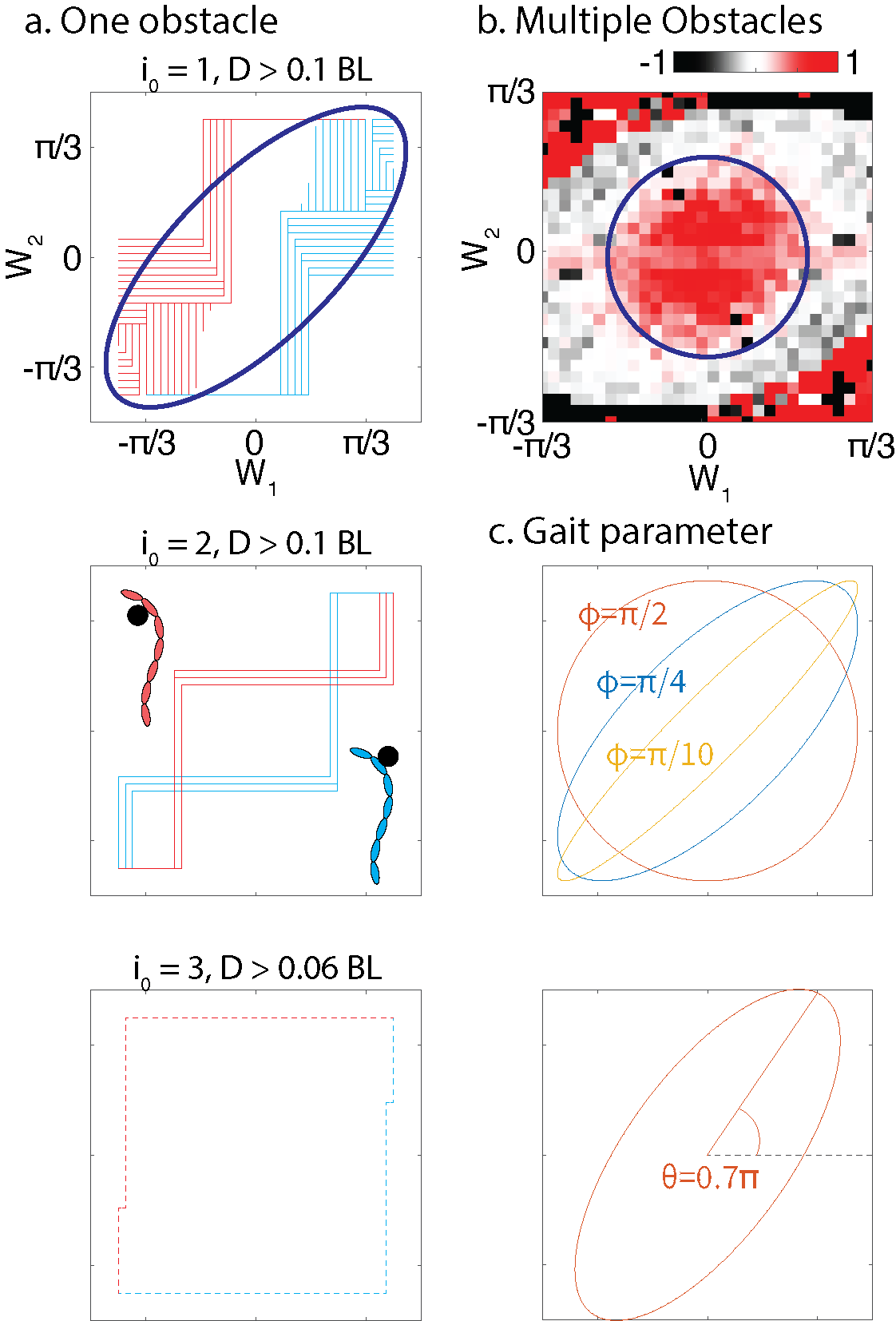}
\caption{\textbf{Identification of gait templates} (a) Collection of effective OAL gaits for (\textit{top}) $i_0=1$, (\textit{mid}) $i_0=2$, and (\textit{bottom}) $i_0=3$. We consider a gait to be effective if it can produce displacement greater than 0.1 BL (body length). Note that there is no effective gait for $i_0=3$. We illustrate the optimal gait with $D=0.05$ for $i_0=3$. (b) Height function for OAL among densely-distributed obstacles. (c) Parameter variation. (\textit{Top}) An illustration of ellipse eccentricity variation by manipulating $\phi$. (\textit{Bottom}) An illustration of ellipse orientation variation by manipulating $\theta$.}
\label{fig:GaitID}
\end{figure}

\subsubsection{\texorpdfstring{${\xi_y}=0$}{ξy=0}} In this case, the robot will remain in contact with the obstacle. If we assume that the friction between the robot and the obstacle is negligible, then the net force from obstacle to robot ($F$) will align with the lateral direction ($y'_i$) of the body frame in link $i_0$. In the body frame of link $i_0$, the interaction between the obstacle and the robot only contributes in the lateral direction. In other words, the force and torque balance in forward and rotational directions are independent from the interactions with obstacles. In this way, we can rewrite Eq. \ref{eq:forceIntegral} into:

\begin{equation}\label{eq:obstacle}
    \bF=\sum_{i} {\Big(\bF^{i}_{\parallel}(    \begin{pmatrix} 
        \xi_x \\
        0 \\
        \xi_\theta
    \end{pmatrix},\bw,\dot{\bw})+\bF^{i}_{\perp}(    \begin{pmatrix} 
        \xi_x \\
        0 \\
        \xi_\theta
    \end{pmatrix},\bw,\dot{\bw})\Big)}=\begin{pmatrix} 
        0 \\
        F \\
        0
    \end{pmatrix}.
\end{equation}

In Eq. \ref{eq:obstacle}, there are two variables and two equality constraints, allowing us to determine the local connection vector field.

\subsubsection{\texorpdfstring{${\xi_y}>0$}{ξy>0}} In this case, the robot will leave the obstacle. In this way, original force and torque balance in Eq.~\ref{eq:forceIntegral} are still valid to determine the local connection vector field.

\subsection{Inequality Constraints}

With the two mutually exclusive interactions conditions, it is thus important to establish a criterion to evaluate the direction of ${\xi_y}$. We first explore the conditions where the robot leaves the obstacle. Specifically, from the equation of motion (Eq.~\ref{eq:EquationOfMotion2}), the lateral velocity ${\xi_y}$ can be approximated by:
\begin{equation}
    \xi_y={\bA_y(\bw)\dot \bw}, \label{eq:conditon}
\end{equation}
\noindent where $\bA_y(\bw)$ is the second row of the local connection matrix $\bA(\bw)$. On the one hand, if ${\bA_y(\bw)\dot \bw}>0$, the robot will leave the obstacle, which is consistent with our assumed condition. In this case, Eq.~\ref{eq:conditon} is valid in accordance with Eq.~\ref{eq:forceIntegral}, where we use condition (2) to determine the local connection matrix. On the other hand, if ${\bA_y(\bw)\dot \bw}\leq0$, the robot will keep engaging with the obstacle, which contradicts our assumption. In this case, Eq.~\ref{eq:conditon} is not valid, and we will use Eq.~\ref{eq:obstacle} and condition (1) to determine local connection matrix.

\subsection{Gait Design}
With the above model, we can now design gaits for limbless robots in obstacle-rich environments. With the optimal gait, the robot should take the best advantage of each obstacle-interaction and leave the obstacle only when necessary. Consider the joint angle limit being $\theta_m$ ($w_1,w_2\in [-\theta_m, \theta_m]$. Let $\Phi= \{ \phi: [0, T]\rightarrow [-\theta, \theta]\times [-\theta, \theta] \}$ be the collection of all paths in the shape space; let $V_1$ be the local connection vector field generated from condition 1 (Eq. \ref{eq:obstacle}); and $V_2=\bA_y(\bw)$. The gait optimization problem becomes a line integral subject to direction constraints:

\begin{problem}\label{prob:opti}
	Find the path $\phi\in\Phi$,
	subject to: $\frac{d\phi(t)}{dt}\cdot V_2\Big(\phi(t)\Big)>0\ \forall \ t\in [0, T]$,
	such that $\int_0^T \frac{d\phi(t)}{dt}\cdot V_1\Big(\phi(t)\Big)\mathrm{d}t$ is maximized.
\end{problem}

Assuming $i_0 = 1$, we showed an example of $V_1$ and $V_2$ in Fig. \ref{fig:VF12}.

\subsection{Numerical Optimization}\label{sec:numOpti}

In practice, we discretize the shape space into a $(n+1)\times (n+1)$ lattice grid, where $n$ is a suitable positive integer. The values of $V_{1}$ and $V_{2}$ are then numerically calculated at the grid points: $V_{i}(x,y) = \bigg[V_{i,1}(x,y), V_{i,2}(x,y)\bigg]$ where $i=1,2$ and $(x,y)$ is a discretized element in the shape space. We optimize $\phi$ among lattice paths with horizontal and vertical line segments. $V_2$ is one part of the vector fields for locomotion in isotropic environment; thus it is reasonable to assume that $V_2$ is a conservative vector field~\cite{liljeback2013lateral,chongmoving,alben2019efficient}. Then we can compute a potential function $P(x,y)$ defined on the shape space such that $V_{2}$ is the gradient of $P(x,y)$.

We consider a weighted directed graph $G=(U,A)$, where the set of vertices $U$ consists of the $(n+1)\times (n+1)$ lattice points\footnote{We chose the letter $U$ (instead of $V$) to represent collections of vertices to avoid notation confusion with $V_{1,2}$ as in vector fields.}. In this way, at each vertex $u=(x,y)\in U$, there are 4 adjacent vertices: $\{(x\pm 1, y)$,\ $(x, y\pm 1)\}$. The arcs are constructed in the following way: 
\paragraph{} If $P(x+1,y) > P(x,y)$, then we add an arc from $(x,y)$ to $(x+1,y)$ with weight $V_{1,1}(x,y)$ to $A$;
\paragraph{} If $P(x-1,y) > P(x,y)$, then we add an arc from $(x,y)$ to $(x-1,y)$ with weight $V_{1,1}(x,y)$ to $A$;
\paragraph{} If $P(x,y+1) > P(x,y)$, then we add an arc from $(x,y)$ to $(x,y+1)$ with weight $V_{1,2}(x,y)$ to $A$;
\paragraph{} If $P(x,y-1) > P(x,y)$, then we add an arc from $(x,y)$ to $(x,y-1)$ with weight $V_{1,2}(x,y)$ to $A$;

Thus, the existence of an arc $a_{ij}\in A$ (from vertex $u_i$ to $u_j$, $u_{i},u_{j}\in U$) indicates that the move from $u_i$ to $u_j$ has positive dot product in $V_2$. The weight of $a_{ij}$ denotes the line integral from $u_i$ to $u_j$ along $V_1$.

\begin{lemma}\label{lem:basis}
    $G$ is a directed acyclic graph (DAG).
\end{lemma}

\begin{proof}[Proof of Lemma \ref{lem:basis}]
     Let $C$ be a directed cycle in $G$. From our previous assumptions, every arc in $C$ has positive dot product in $V_2$. Thus, the sum of all dot product of arcs in $C$ and $V_2$ must be strictly positive. This indicates that there exists a path in a conservative vector field ($V_2$) with positive strictly line integral, which violates our assumption. Therefore, there is no directed cycle in $G$.
\end{proof}

With the aforementioned notation, a discretized version of Problem \ref{prob:opti} becomes

\begin{problem}\label{prob:disc}
	Find a simple directed path in $G=(U,A)$ with maximal weight.
\end{problem}

It is well-known that Problem \ref{prob:disc} in a DAG has a linear-time algorithm if the starting point is fixed~\cite[p. 661]{sedgewick2011algorithms}. So we can run this algorithm once for each vertex in $U$ to solve Problem~\ref{prob:disc}. Since $|U|=(n+1)^{2}$, our algorithm has time complexity $O(n^{4})$.

We implemented this algorithm in MATLAB and found optimal paths in our lattice grid.

\subsection{Gait Identification}

From the algorithms introduced in Sec. \ref{sec:numOpti}, we solve Problem \ref{prob:opti} and identify the effective gait paths $\phi_{LHS}$ with link of contact varying from 1 (head) to 3 (mid-body) in Fig.~\ref{fig:GaitID}a. We define a gait path to be effective if it can cause net displacement greater than 0.1 body length (BL). Note that $\phi_{LHS}$ (colored red) denotes gait paths designated for robot interacting with an obstacle on the left-hand-side. From symmetry, we can identify $\phi_{RHS}$ with an obstacle on the right-hand-side (colored blue). Note that no gait path can lead to displacement higher than 0.1 BL when interacting with obstacles on mid-body links ($i_0 = 3$). In Fig.~\ref{fig:GaitID}a (bottom), we illustrate the optimal gait path which causes displacement of 0.06 BL.

From Fig.~\ref{fig:GaitID}a, we notice that the number of effective gait paths decreases as the link of contact transitions from head to mid-body links. Further, the properly designed gait path can cause up to 0.35 BL (per cycle) when interacting with the head link; whereas it can only cause 0.12 and 0.06 BL when $i_0$ changes to 2 and 3 respectively. Therefore, our results indicate that it is desired to interact with obstacles from head link rather than mid-body links.

We further observe from Fig.~\ref{fig:GaitID}a (top) that almost all effective gait paths emerge to be (at least a part of) elliptical paths. To quantify this observation, we fit the collection of effective gaits with an oriented ellipse. An ellipse with flatness (defined as the ratio of short-axis and long-axis) around $0.5$ can reasonably fit effective gait paths. The ellipse oriented at angle of $\pi/4$ with respect to the horizontal axis. 


\section{Modeling Interaction with Multiple Obstacles}\label{sec:multiObs}

Now we consider multiple obstacles in contact with the robot. Similar to our analysis before, there are three conditions with respect to the status of robot leaving/engaging obstacles:

\paragraph{Robot only interacts with one obstacle} In this case, the robot will only remain contact with one of the obstacle. This condition is similar to condition (1) in Sec.~\ref{sec:model}.

\paragraph{Robot leaves all obstacles} In this case, the robot will leave all obstacles, which is similar to condition (2) in Sec.~\ref{sec:model}. 

\paragraph{Robot interacts with multiple obstacles} In this case, the robot will remain contact with more multiple obstacles. As illustrated in Fig.~\ref{fig:VF12}c, the presence of multiple obstacles restricts the lateral oscillation and rotational oscillation of the central body axis on robot (assuming the friction is negligible~\cite{liljeback2011experimental}). The definition of central body axis frame can be found in ~\cite{gong2016simplifying,rieser2019geometric}. In other words, in the body reference frame of central body axis, we have:

\begin{equation}\label{eq:multiobstacle}
    \bF=\sum_{i} {\Big(\bF^{i}_{\parallel}(    \begin{pmatrix} 
        \xi_x \\
        0 \\
        0 \\
    \end{pmatrix},\bw,\dot{\bw})+\bF^{i}_{\perp}(    \begin{pmatrix} 
        \xi_x \\
        0 \\
        0 \\
    \end{pmatrix},\bw,\dot{\bw})\Big)}=\begin{pmatrix} 
        0 \\
        F_y \\
        F_\tau \\
    \end{pmatrix}.
\end{equation}

In Eq. \ref{eq:multiobstacle}, there is only one variable and one equality constraint, allowing us to determine the local connection vector field. 

Note that the condition determination for when the robot is in contact with multiple obstacles can be challenging, which likely requires sensing and compliance as indicated in prior work~\cite{travers2018shape,lillywhite2014snakes}. However, consider the case where the obstacles are so densely distributed that the robot will inevitably interact with multiple obstacles. In this case, we can simply assume that condition (c) is always valid and calculate the height functions to determine the optimal gaits. We illustrate the height function in Fig.~\ref{fig:GaitID}.b. We notice that a traveling-wave gait path emerges as an optimal gait in environments with densely-packed obstacles.

\begin{figure}[t]
\centering
\includegraphics[width=1\linewidth]{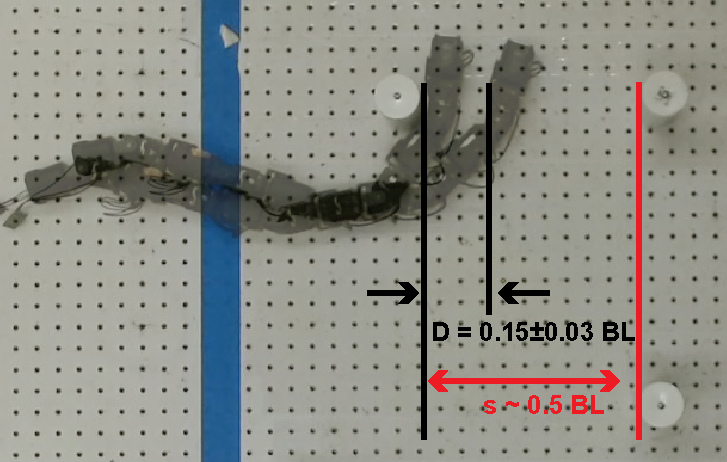}
\caption{\textbf{Minimal obstacle spacing for effective OAL} Robot effective OAL gaits through interaction with an obstacle. However the translation ($D=0.15\pm0.03 BL$) is not sufficient to reach the next obstacle (spacing $s\sim 0.5 BL$). Therefore, the robot will stuck at the gap between two obstacles.
}
\label{fig:spacing}
\end{figure}

\section{Robophysical model}

In robophysical experiments, we used a limbless robot composed of 11 identical alternative pitch-yaw arranged rotary joints using Dynamixel AX-12a motors. 
The gaits are executed by controlling the positions of joints to follow a sequence of joint angle commands. Note that for 2D in-plane motion, we only command odd (yaw) joints to move while the even (pitch) joint angles are held at zero. For each gait tested, we repeat the experiment at least six times. 
In each trial, we commanded the robot to execute three cycles of the gait.
The motion of the robot is tracked by an OptiTrack motion capture system at a 120 FPS frequency with eight reflective markers affixed along the midline of the robot.

\section{Results}

\subsection{Shape basis function optimization for OAL}

To verify our prediction on shape basis function for OAL, we conducted robophysical experiments using parallel walls. As shown in Fig. \ref{fig:lie}, the robot was confined between two parallel smooth walls with spacing 0.3 body length of the limbless robot. The interactions between the robot and the wall then restricted the velocity of the robot to the forward direction. Thus the average speed of the robot in parallel walls closely resembles the forward velocity integral from our geometric mechanics analysis. Interestingly, the robot has negative forward displacement, in agreement with the predictions from the forward velocity integral. Further, we noticed that highest backward speed occurs at $f_s=0.5$, also in agreement with our theoretical predictions. Therefore, robophysical experiments verified our prediction that lateral forces can also contribute to forward velocities via Lie bracket effects. In most cases, the interaction between obstacles and the robot is predominantly in lateral directions~\cite{liljeback2011experimental}. We thus chose $f_s=0.5$ in our later analysis.

\subsection{Minimal obstacle spacing for effective OAL}

We investigate the minimal obstacle spacing for effective OAL. From our analysis in Sec.~\ref{sec:model}, we predicted that with proper coordination, the interaction between a robot and a (single) obstacle can cause displacement up to 0.35 BL. Note that the number 0.35 is computed based on the robot morphology and our choice of shape basis function. In other words, there exists a upper bound on how much a single obstacle can contribute to robot OAL. If the obstacle spacing is greater than such a bound, the robot will be likely unable to reach the next surrounding obstacle.

To verify our prediction, we tested robot OAL performance in obstacle-rich environments with spacing greater than 0.35 BL. We notice that while OAL gaits can cause some finite displacement through the interaction with the first obstacle, such translation is not sufficient to reach the next obstacle (Fig.~\ref{fig:spacing} and SI video). In this way, the robot will get ``stuck" in the gap between two obstacles. 

\subsection{OAL with sparsely distributed obstacles}

\begin{figure}[t]
\centering
\includegraphics[width=1\linewidth]{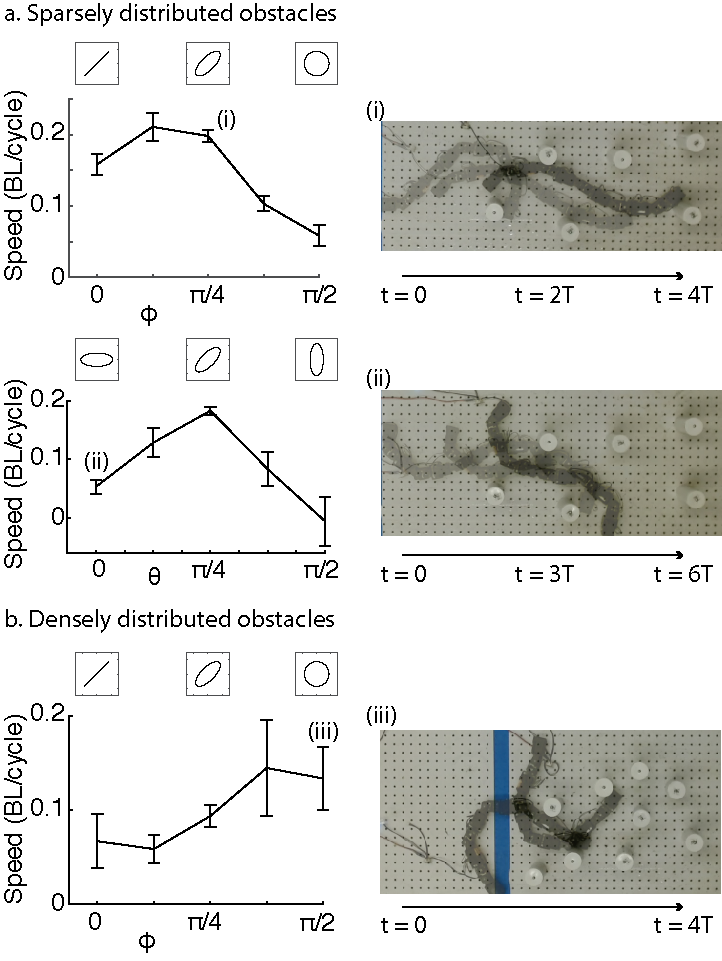}
\caption{\textbf{Robophysical OAL experiments} (a) Sparsely distributed obstacles. (\textit{a.Top}) OAL performance as a function of $\phi$ (for fixed $\theta=\pi/4$). Elliptical gaits ($\phi\sim\pi/4$) leads to the best OAL performance. (i) Snapshots of robot execute elliptical gaits ($\phi=\pi/4$) among sparsely distributed obstacles. (\textit{a.Bottom}) OAL performance as a function of $\theta$ (for fixed $\phi=\pi/4$). Elliptical orientation ($\theta=\pi/4$) lead to the best OAL performance. (ii) Snapshots of robot execute uncoordinated elliptical gaits ($\theta=0$) among sparsely distributed obstacles. (b) Densely distributed obstacles. OAL performance as a function of $\phi$. Circular gaits ($\phi=\pi/2$) leads to the best OAL performance. (iii) Snapshots of robot execute traveling-wave gaits ($\phi=\pi/2$) among densely distributed obstacles.
}
\label{fig:results}
\end{figure}

From our framework, we predicted that elliptical gaits can have the best performance among sparsely distributed obstacles. To test our prediction, we constructed a sparsely-distributed obstacle-rich environments. The obstacles are randomly positioned in the track (Fig.~\ref{fig:results}). We then conduct robophysical experiments and evaluate the OAL performance of various gaits.

\subsubsection{Varying ellipse eccentricity}

We first test gaits with varying eccentricity. Specifically, prescribe the reduced shape variable by $w_1(t) = w_m\sin{(\omega t)},\ w_2(t) = w_m \cos{(\omega t + \phi)}$, where $\omega$ is the temporal frequency, $w_m$ is the amplitude, and $\phi$ controls the eccentricity. As illustrated in Fig.~\ref{fig:GaitID}.c, varying $\phi$ can facilitate the transition from standing wave ($\phi= 0$) to traveling wave ($\phi= \pi/2$) in the shape space. In our theoretical analysis (Sec.~\ref{sec:model}), we predict that $\phi=\pi/4$ can have the best OAL performance. We test gaits with different $\phi$ among sparsely-distributed obstacles. We notice that $\phi=\pi/4$ indeed outperforms other gaits, including standing wave and traveling wave (Fig.~\ref{fig:results}). 

To explore the principle behind the advantage of the elliptical gaits, we measured the duration of obstacle-contact in these experiments. Here, we defined the duration of contact by the average fraction that the robot is interacting with obstacles $\tau/T$, where $\tau$ is empirically measured average contact duration (Fig. \ref{fig:multiplePosts}) and $T$ is the gait period. We notice the contact duration in the standing-wave gait is significantly lower than the elliptical-wave and traveling-wave gaits, indicating that the standing-wave gait has the lowest duration of beneficial contact between robot and obstacle. We also measured the attack angle between the robot and the obstacle. It is defined as the angle between the head link and the obstacle at the end of the robot-obstacle interaction. As posited by \cite{liljeback2011experimental}, larger attack angle indicates greater push from the obstacle to robots. As shown in Fig. \ref{fig:results}, the attack angles in the traveling-wave gait are significantly lower than the elliptical and standing wave gaits, indicating that the traveling wave gait can take the least advantage of the obstacle.

\begin{figure}[t]
\centering
\includegraphics[width=1\linewidth]{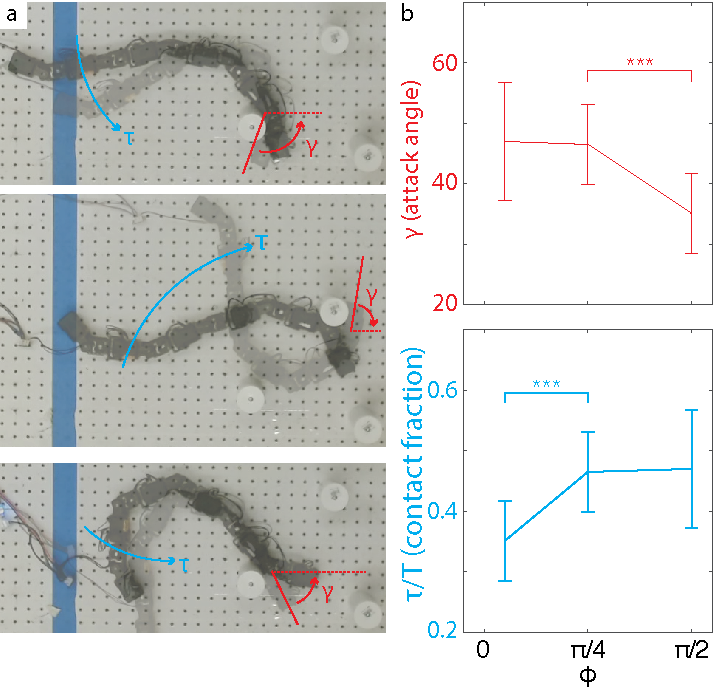}
\caption{\textbf{Advantage of elliptical gaits} (a) Snapshots of robots executing (\textit{top}) standing wave, (\textit{mid}) elliptical wave, and (\textit{bottom}) traveling wave locomotion among sparsely-distributed obstacles. Attack angle and contact duration are labelled. (b) (\textit{top}) Attack angle as a function of $\phi$. Traveling wave ($\phi = \pi/2$) have significantly lower attack angle than standing wave ($\phi = 0$) and elliptical wave ($\phi = \pi/4$). (\textit{Bottom}) Contact fraction as a function of $\phi$. Standing wave have significantly lower attack angle than traveling wave and elliptical wave. 
}
\label{fig:multiplePosts}
\end{figure}

\subsubsection{Varying ellipse orientation}

We further explore the optimal ellipse eccentricity. Consider an elliptical gait with $\phi=\pi/4$. We define ($\theta$) as the angle between the long axis and the horizontal axis. We illustrate an example gait with $\theta = \{0.45\pi,\ 0.7\pi\}$ in Fig.~\ref{fig:GaitID}c. From our theoretical analysis (Sec.~\ref{sec:model}), we predict that $\theta= \pi/4$ can cause the optimal OAL performance. Robophysical experiments verified our prediction that $\theta=\pi/2$ causes the best OAL performance. 

\subsection{OAL with densely distributed obstacles}

\begin{figure}[t]
\centering
\includegraphics[width=1\linewidth]{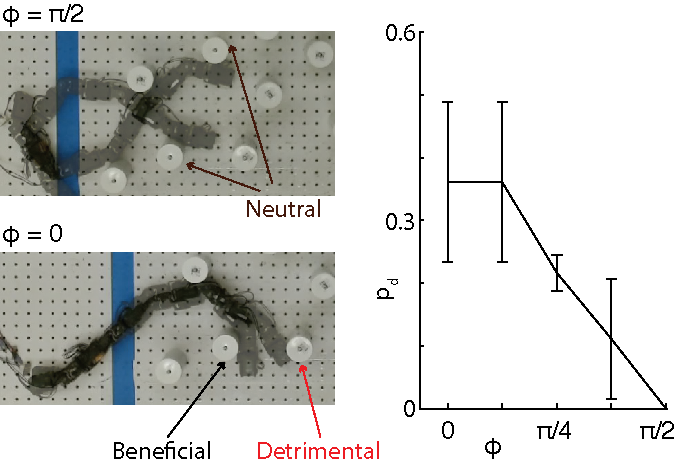}
\caption{\textbf{Beneficial obstacles.} (\textit{Left}) Snapshots of traveling wave  (\textit{top}) and standing wave (\textit{bottom}) locomotion among densely-distributed obstacles. Beneficial, detrimental, and neutral obstacles are labeled. (\textit{Right}) $p_d$, the probability of encountering detrimental obstacles, plotted as a function of $\phi$. $p_d$ decreases as $\phi$ increases.
}
\label{fig:stuck}
\end{figure}

We next explore OAL among densely distributed obstacles. We constructed a densely-distributed obstacle-rich environments where robot will inevitably encounter with multiple obstacles. We tested gaits with varying $\phi$ on densely-distributed obstacles and observed that traveling-wave gaits ($\phi = \pi/2$) can cause the best OAL performance (Fig.~\ref{fig:results}). We acknowledge that to effectively navigate in environments with many obstacles, sensing and/or compliance is typically required~\cite{wang2020directional,travers2018shape}. Since our analysis in Sec.~\ref{sec:multiObs} is limited to open-loop gait-level design, the large variation in our experiments is expected. 

To explore the physical principles behind the advantage of traveling-wave gaits, we examine the interaction profile between the robot and the obstacles. As predicted in our theoretical analysis (Sec.~\ref{sec:multiObs}), effective OAL in traveling-wave gaits results from the combined effects of multiple obstacles restricting the lateral/oscillation of central body axis. Therefore, there is no clear definition of ``beneficial" or "detrimental" obstacles in traveling-wave gaits. As illustrated in Fig.~\ref{fig:stuck}, interactions between the robot and obstacles are mostly perpendicular to the direction of motion (therefore considered as ``neutral") in traveling-wave gaits. On the other hand, effective OAL for elliptical and standing wave relies more on the interaction with a single obstacle. Therefore, OAL performance of elliptical and standing wave are sensitive to the distribution of obstacles. 

Following this idea, we record the probability of robot interacting of ``detrimental" obstacle ($p_d$) for traveling, elliptical, and standing gait templates. We notice that $p_d$ increases as $\phi$ increases (Fig.~\ref{fig:stuck}). Moreover, once interacting with the detrimental obstacles, the probability of escaping decreases as $\phi$ decreases (Fig.~\ref{fig:stuck}). 

\section{Discussion and Conclusion}

In this paper, we expanded the scope of geometric mechanics to heterogeneous environments. Specifically, we established a novel model that maps the presence of an obstacle in position space to constraints in shape space. In doing so, we illustrate that (1) there exists a threshold obstacle spacing below which OAL is likely not effective; (2) lateral forces from obstacles can also contribute to forward displacement via the Lie bracket effect; (3) elliptical-wave gaits ($\phi = \pi/4$, $\theta = \pi/4$) are specialized for locomotion among sparsely-distributed obstacles; and (4) traveling-wave gaits ($\phi = \pi/2$) are specialized for locomotion among densely-distributed obstacles. Our predictions are verified in robophysical experiments.

This paper focused on the open-loop gait-level design for OAL. We acknowledge that gait-level design is in general not sufficient for effective OAL, especially among densely-distributed obstacles. However, we believe that proper gait design can simplify necessary controls (e.g., passive body dynamics~\citet{wang2023lattice} and sensor-based feedback controls) for OAL. In concurrent work, it is illustrated that with the help of passive body dynamics, our framework helps facilitate effective OAL in various obstacle-rich environments. 

Apart from forward locomotion, our framework can also apply to studying turning behavior in obstacle-rich environments. For example, prior work indicates that the omega turn~\cite{wang2020reconstruction} emerges to be a robust turning gait, especially in obstacle-rich environments; \citet{wang2022omega} show that the presence of obstacles can even aid turning behaviors in limbless robots. We suspect that omega turn gaits can benefit from interaction with obstacles because of the kinematic properties of their gait trajectories. In future work, we aim to use our OAL framework to analyze the turning behaviors in limbless robots. In doing so, our framework paves the way toward machines that can traverse complex environments and facilitates understanding of biological locomotion.

\newpage

\bibliographystyle{plainnat}
\bibliography{references}

\end{document}